\RequirePackage{fix-cm}
\documentclass[twocolumn]{svjour3}          
\smartqed  
\usepackage{extarrows}
\usepackage{amsfonts}  
\usepackage{graphicx}
\usepackage{subfigure}
\usepackage{amsmath}
\usepackage{authblk}
\usepackage{overpic}
\usepackage{tikz}
\usetikzlibrary{backgrounds,fit,shapes,calc}
\usepackage[colorlinks,citecolor=blue]{hyperref}
\usepackage{pgfplots,filecontents}
\usepgflibrary{plotmarks}
\usepackage{animate}
\usepackage[normalem]{ulem}
\newcommand{\stkout}[1]{\ifmmode\text{\sout{\ensuremath{#1}}}\else\sout{#1}\fi}

\usepackage{comment}
\usepackage{color}
\usepackage{soul}
\soulregister\cite7
\soulregister\ref7
\soulregister\pageref7

\usepackage[noadjust]{cite}

\newcommand\norm[1]{\left\lVert#1\right\rVert}
\title{Eigenbackground Revisited: Can We Model the Background with Eigenvectors?}

\author{Mahmood Amintoosi      \and
        Farzam Farbiz 
}

\institute{Mahmood Amintoosi \at
              Faculty of Mathematics and Computer Science, Hakim Sabzevari University, Iran \\
              \email{m.amintoosi@hsu.ac.ir}           
           \and
           Farzam Farbiz \at
              A*STAR Institute of High Performance Computing (IHPC), 1 Fusionopolis Way, \#16-16 Connexis, Singapore 138632 \\
              \email{Farzam\_Farbiz@ihpc.a-star.edu.sg}   
}

\date{Received: date / Accepted: date}

\begin{document}

\maketitle

\begin{abstract}
Using dominant eigenvectors for background modeling (usually known as Eigenbackground) is a common technique in the literature. However, its results suffer from noticeable artifacts.  Thus, there have been many attempts to reduce the artifacts by making some improvements/enhancements in the Eigenbackground algorithm.
In this paper, we show the main problem of the Eigenbackground is at its own core and in fact, it may not be a good idea to use the strongest eigenvectors for modeling the background. Instead, we propose an alternative solution by exploiting the weakest eigenvectors (which are usually thrown away and treated as garbage data) for background modeling.  MATLAB  codes are available at the GitHub of the paper
\footnote{\url{https://github.com/mamintoosi/Eigenbackground-Revisited}}.
\end{abstract}


\section{Introduction}

Background segmentation is one of the fundamental tasks in computer vision with a wide spectrum of applications from video compression to scene understanding. A very commonly used example of background segmentation is on detecting moving objects in videos taken from static cameras, by finding the differences between the new frame and the background model of the scene (or the reference frame).

There have been several background subtraction and modeling methods in the literature. All of these methods aim to effectively estimate the background model from a temporal sequence of  video frames. One of the well-known methods in this field is the Eigenbackground method \cite{Oliver2000PCA}, which  models the background by a set of dominant eigenvectors. 
The idea has been used in many papers including 
\cite{citeulike:8163606,skocajIMAVIS08,Dickinson20091326,Yuan20092450,Casares20101223,Dong201131,Tzevanidis2011105,
guyon:hal-00811439,Vosters20121004,Krishna2012624,Zhao20121134,Yeo20131583,
Seger20142098,Spampinato201474,Bouwmans201431,
Varadarajan20153488,DBLP:journals/corr/ShakeriZ15,Dou2015382,conf/pcm/XuSG06,
DBLP:journals/cviu/ChenTWH16,
Wan2018Total,
Banu2020SC,Djerida2020}
for background segmentation. In some literature, this method has been also called as PCA and Subspace Learning.

According to  \cite{Oliver2000PCA} the core idea of the Eigenbackground method is as follow:
\begin{quotation}
	``In order to reduce the dimensionality of the
	space, in principal component analysis (PCA) only $M$
	eigenvectors (eigenbackgrounds) are kept, corresponding to
	the $M$ largest eigenvalues to give a $\Phi_M$ matrix.''
\end{quotation}

The advantage of this method is that it is a mathematically straightforward algorithm, which requires no heuristics or parameters that must be set manually.
However, most of the research articles that used this method highlighted its rather poor quality result in background modeling especially in the presence of some foreground objects in the input video frames. see figures \ref{fig:Dou2015382_Fig8}  and \ref{fig:Dong201131_Fig4} as examples of the Eigenbackground's performance in modeling the background.

In \cite{Oliver2000PCA} it is claimed that the eigenspace provides a model for the background:

\begin{quotation}
``\dots the eigenspace
provides a robust model of the probability distribution
function of the background, but not for the moving objects''

\end{quotation}
In contrast to this statement, in this paper
 we will show that the eigenspace is indeed influenced by moving objects, not by the static parts of the scene. This is actually the same as the concept of eigenfaces, which models the faces and not the background.

 \begin{figure}[t]
 \centering
\includegraphics[width=1\columnwidth]{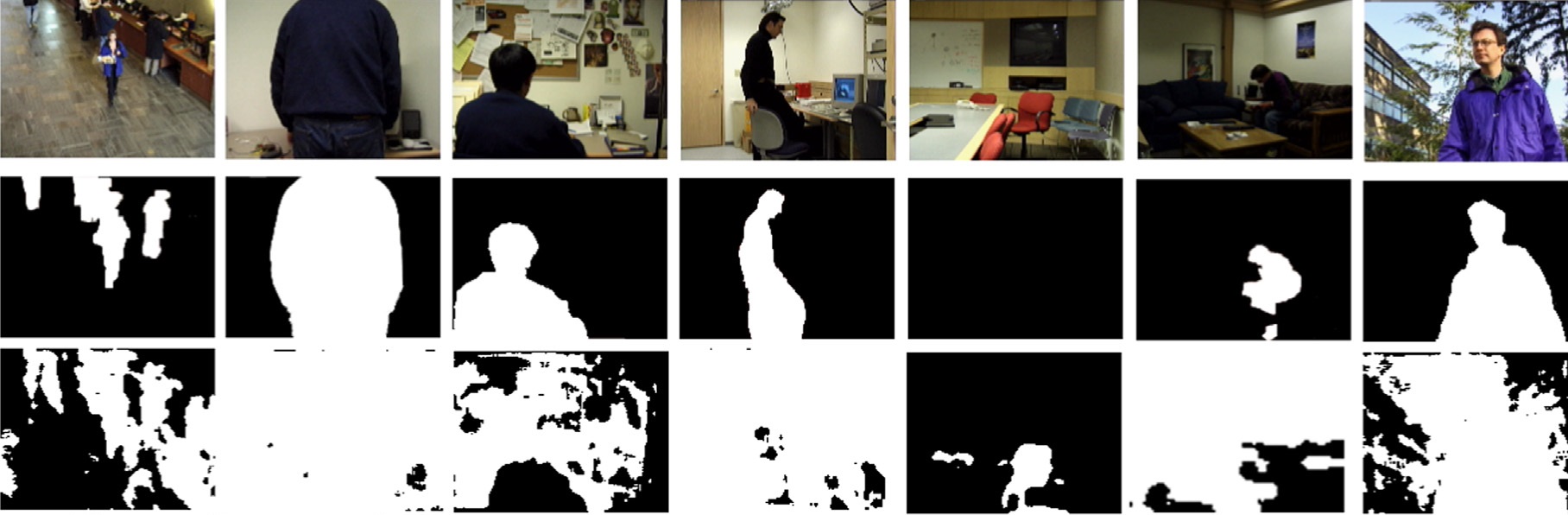}
\caption{The result of the Eigenbackground (Fig. 8 of \cite{Dou2015382}) in background modeling. From top to bottom, the original images, ground truth and the results of  foreground detection using background models created by the Eigenbackground method (courtesy from \cite{Dou2015382}).}
\label{fig:Dou2015382_Fig8}
\end{figure}

\begin{figure}[t]
\centering
{\label{fig::Dong201131_Fig4:a}
\includegraphics[width=1\linewidth]{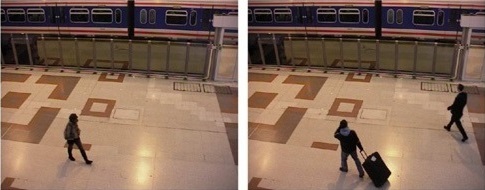}}\\
{\label{fig::Dong201131_Fig4:b}
\includegraphics[width=1\linewidth]{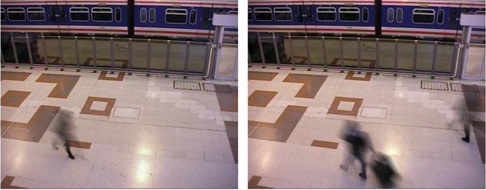}}\\
\caption{The low quality of the Eigenbackground method in background modeling due to the presence of some foreground objects in the scene (Fig. 4 of \cite{Dong201131}, curtesy from  \cite{Dong201131}). Top row: Input images. Bottom row: Created background model.}

\label{fig:Dong201131_Fig4}
\end{figure}

The rest of this paper is organized as follows.
In Section \ref{sec:review} we review the Eigenbackground algorithm by looking into the literature review and study on the core problem related to this method. We also propose an alternative method to the original Eigenbackground method in this section.
 Next, we discuss the theoretical aspects of the eigenspace in  Section \ref{Sec:Math}. Further examples and experiments are demonstrated in Section \ref{Sec:Exp}. The conclusions are provided in Section \ref{Sec:Conc}. Finally, we explain some of the lemmas used in the paper in the Appendix.

In summary, the contributions of this paper are as follows:
\begin{enumerate}
\item To prove why using the Strongest EigenVectors (SEV) in the Eigenbackground algorithm is not suitable for background modeling
\item To propose an alternative to Eigenbackground, by using the Weakest EigenVectors (WEV)
 for modeling the background of a video frame.
\end{enumerate}

\section{Review of the Eigenbackground Algorithm}\label{sec:review}

Eigenbackground represents the background as a set of dominant eigenvectors, extracted from a set of training images. As mentioned in \cite{Oliver2000PCA}:
\begin{quotation}
``Note that moving objects, because they don't appear in
the same location in the N sample images and they are
typically small, do not have a significant contribution to this
model. Consequently, the portions of an image containing a
moving object cannot be well-described by this eigenspace
model (except in very unusual cases), whereas the static
portions of the image can be accurately described as a sum
of the various eigenbasis vectors. That is, the eigenspace
provides a robust model of the probability distribution
function of the background, but not for the moving objects.''
\end{quotation}
The dominant eigenvectors are then used to estimate static parts of the scene to form the background image in this method.

However, there has been no proof for the above assumption. Although the moving objects might cover only small parts from the video frames, it does not mean they will not affect the dominant eigenvectors in a significant way.

To test the above assumption, which is indeed the core of the Eigenbackground algorithm, let's review the eigenvectors in a matrix.
For this review, we use the concept of Principal Component Analysis (PCA) as its eigenvectors are the basis vectors of the eigenspace.

It is well-known that the first principal component has the largest possible variance (that is, accounts for as much of the variability in the data as possible), and each succeeding component in turn has the highest variance possible under the constraint that it is orthogonal to the preceding components. Knowing PCA components are calculated using eigenvectors of the matrix, the most dominant eigenvector should also be in the direction with the largest possible variance.  

For background modeling of a video data, the variance is expected to be higher for the frames with foreground objects (as they are likely to be very different  from each other) compared to the frames that mostly contain background. Thus, the strongest eigenvector will be affected by foreground frames. 
 Consequently, we should expect some ghosting effect in the background models generated by the Eigenbackground method as it is based on the strongest eigenvectors.

Figure \ref{fig:Dou2015382_Fig8} from \cite{Dou2015382} shows some examples of using the Eigenbackground method in background modeling and foreground objects detection. However, as can be seen in this figure, the results of foreground objects detection are poor. This was due to the presence of some foreground objects in every  video frame leading to have poor background models. As another example, Figure \ref{fig:Dong201131_Fig4} from \cite{Dong201131} shows low quality of the background model due to foreground objects.

The low performance of Eigenbackground method has highlighted in the literature too.  For instance, Vosters et al. mentioned in \cite{Vosters20121004}, ``It is imperative the background images in the training set do not contain any foreground objects, since they would cause significant errors in the reconstructed background''.  As another example, Shah et al. in \cite{Shah2017} stated that ``Based on the extensive literature survey conducted, it has been observed that, even though the Eigenbackground subtraction method is computationally efficient, the motion detection is not as accurate as that of other popular motion detection algorithm''. Bouwmans in \cite{Bouwmans2009SubspaceLF} also highlighted limitation of the Eigenbackground method in dealing with large size foreground objects.

Nevertheless, the focus on previous works in this field was mainly on applying different techniques to improve/enhance the performance of the original Eigenbackground method. For example, Cao et al. in \cite{Cao2008} and Ziubinski et al. in \cite{Ziubi_ski_2014}  used small blocks for eigenvector analysis as there is higher chance for each block to be free from any foreground objects in some of video frames. Furthermore, changes in the image illumination would be more uniform for a small block (compared to the whole image frame) and hence such illumination changes will not significantly affect the background modeling result. As another example, Kim et al. in \cite{Kim_2013}  constructed an Eigenbackground model using a selected eigenvectors and then used this background model as input data for an MoG model to generate the final background model. Similarly, Hughes et al. in \cite{Hughes2013} used a subset of eigenvectors  and then improved the result using bootstrapping training.

Some other works such as \cite{citeulike:8163606,guyon:hal-00811439} tried to use different variations of PCA, including Robust PCA, to improve the performance of Eigenbackground in background  subtraction. However, if the training set does not include any foreground objects as suggested in  \cite{Dong201131}, then there is no need to use PCA or Eigenbackground algorithm to model the background and a simpler background modeling with less processing time is sufficient. Indeed, some technical papers use the idea of combining the mean image and the first significant eigenvectors of PCA for background modeling (see \cite{Bouwmans201431} and \cite{conf/pcm/XuSG06}).

The aim of this paper is to study the core problem associated with the Eigenbackground technique. To do so, let's analyze the eigenvectors for a test video. Figure \ref{fig:SampleVideo} shows a sample highway video with 121 frames. We took a $40\times40$ block of this video  and calculated the  strongest and weakest eigenvectors shown in Figure \ref{fig:highway:40}. As can be seen, and as expected, the strongest eigenvectors are heavily influenced by foreground objects in the scene.

\begin{figure}[t]
\centering
\includegraphics[width=1\linewidth]{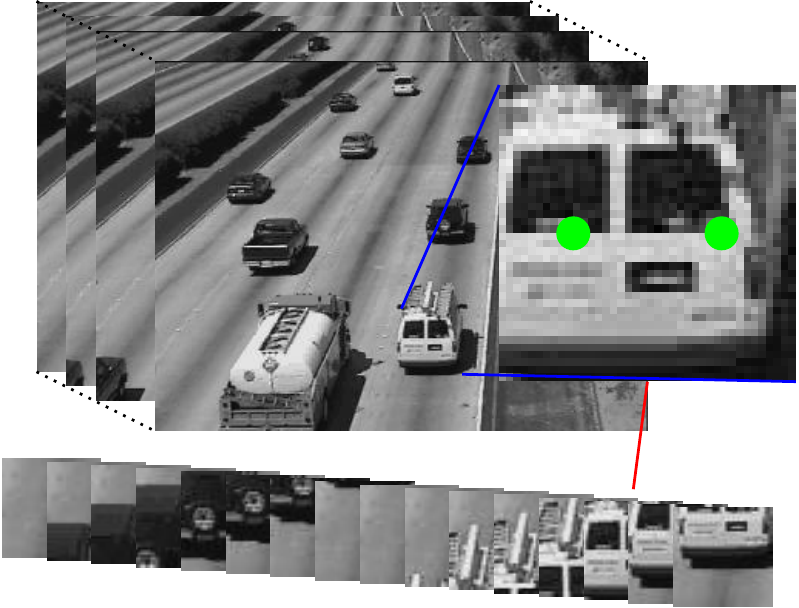}
\caption{Sample frames from a highway video scene and the selected block and pixels for Eigenvector analysis.}
\label{fig:SampleVideo}
\end{figure}

\begin{figure}[t]
\centering
\includegraphics[width=0.9\linewidth]{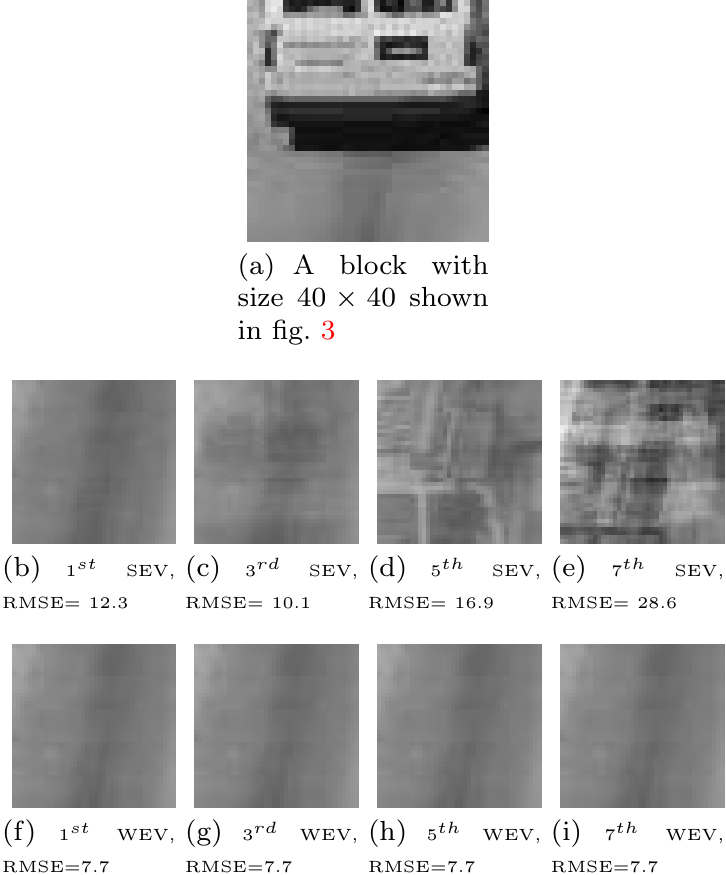}
\caption{Background model created by reconstructing of the projection of a $40\times40$ block of frame 16 (a), onto the each of the first  {1,3,5,7}  and the last eigenvectors with their corresponding RMSE from the ground-truth background model.  
}
\label{fig:highway:40}
\end{figure}

To further analyze how foreground and background objects affect the eigenvectors, we took the two green pixels from this block so we can display every frame for these two pixels in a 2D plot as shown in Figure \ref{fig:PC1and2}. In this figure, blue crosses are for background frames (92 frames from the total of 121 frames) and red circles are for foreground frames (29 frames)
\footnote{For better demonstration, the frames' order is perturbed.}.
 There are four plots in this figure that shows how the green line representing the most dominant eigenvector is changed based on the foreground data to the direction that has the maximum variance.
As can be seen, and will be mathematically explained in Section \ref{Sec:Math}, almost all background instances are located at one point, while the first principal component is significantly influenced by foreground instances.

Since the main eigenvector tends to be in the direction that maximizes the variance for input data, and each succeeding component in turn has the highest variance possible under the constraint that it is orthogonal to the preceding components, it would be logical to assume the least significant eigenvectors might be the right candidate to model the background image frames. The background image frames are not exactly equal to each other so when we put all frames in one $m\times n$ matrix ($m\ge n$), the matrix rank will be the same as the number of frames. But the differences between background frames are much less than foreground ones, and all background frames can be considered as a noisy version of the background model with some small variances. This is similar to the concept of Mixture of Gaussian (MoG)\cite{Stauffer1999Adaptive}, where it is assumed pixels with less variance in video frames likely belong to the background and those with higher variance likely belong to foreground objects. Here we can also consider least significant eigenvectors represent background regions.

\begin{figure}
\centering
\includegraphics[width=1\linewidth]{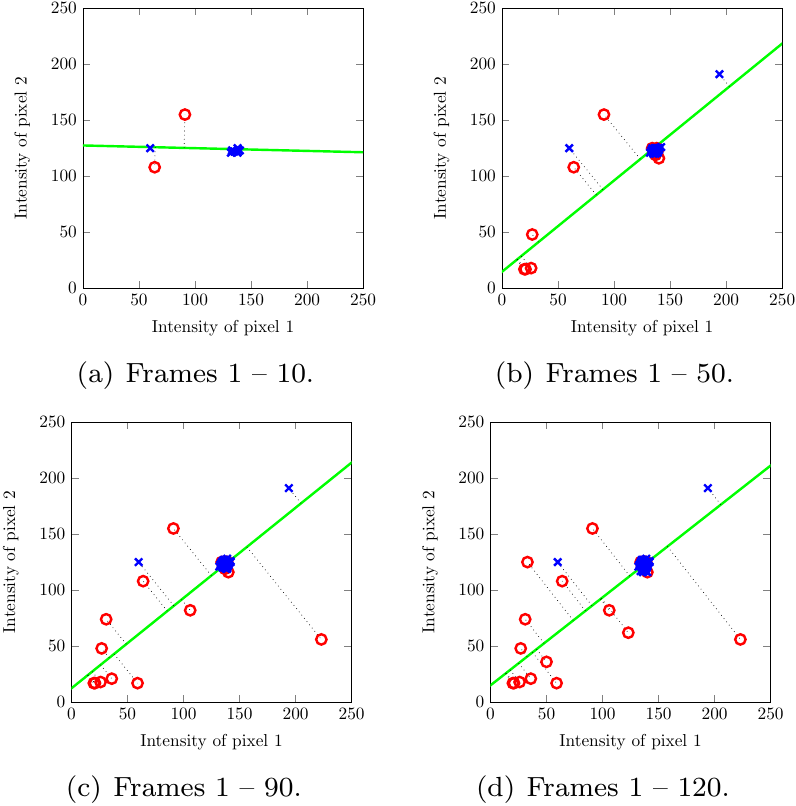}
\caption{Change of dominant PCA vector  in response to new input frame for two green dotted pixels shown in fig \ref{fig:SampleVideo}. The results are calculated and shown for the first 10, 50, 90, and 120 frames.  The background frames are shown by blue crosses and the foreground frames by red circles. The first PCA vector is in the direction that input frames have maximum variance, which is significantly influenced  by foreground frames.
}
\label{fig:PC1and2}
\end{figure}

To test this idea, we used the \textit{highway}  video in Figure \ref{fig:SampleVideo}, split it into $40\times40$ blocks, and calculated  the background using i) 10 strongest and ii) 10 weakest eigenvectors. The results for frame No. 16 are shown in figures \ref{fig:PCBG} and \ref{fig:NonPCBG}. As can be seen in this figure, the background model is better reconstructed using the weakest eigenvectors.

\section{Theoretical Aspects  of Eigenspace}\label{Sec:Math}

The main idea of this paper is to demonstrate the eigen-space produced by the strongest eigenvectors is not a good candidate for the background subspace.
In this section we demonstrate this by showing that the direction of the strongest eigenvector in PCA is influenced by foreground instances.
We start with  some preliminary assumptions and theorems that are needed for this mathematical discussion. 

Let $X = \{x_1,\dots,x_n\}$ be the set of $n$ observations with the dimensionality of $D$. The goal of PCA is to project  the data onto a space having dimensionality of  $M<D$ while maximizing the variance of the projected data. Consider the case for one-dimensional space ($M=1$), where the data is projected into a single unit vector $e$. The goal here is to find a vector $e$ such that the projection of  $x \in X$  along this direction has the maximum variance. Each data point $x_k$ is projected onto a scalar value $e^Tx_k$. The mean of the projected data is $e^T\mu$, where $\mu$ is the sample set mean. 
The variance of the projected data is given by:
\begin{align}
\frac{1}{n}\sum_{k=1}^n(e^Tx_k-e^T\mu)^2 =  \frac{1}{n}e^TSe
\end{align}
where $S$ is the scatter matrix:
\begin{align}
S = \sum_{k=1}^n (x_k-\mu)(x_k-\mu)^T  \nonumber
\end{align}
Using the Lagrange multipliers method, the eigenvector corresponding to the largest eigenvalue of $S$, maximizes $e^TSe$.
\begin{figure}[t]
\centering
\subfigure[Frame No. 16]
{\label{fig:NonPCBG:Frame}
\includegraphics[width=.31\linewidth]{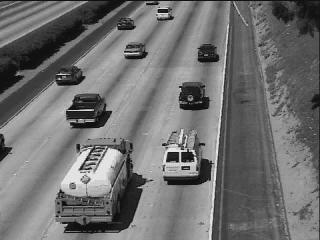}}
\subfigure[Background]
{\label{fig:PCBG:BG}
\includegraphics[width=.31\linewidth]{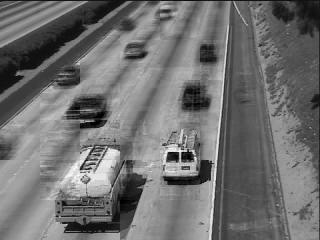}}
\subfigure[Foreground]
{\label{fig:PCBG:FG}
\includegraphics[width=.31\linewidth]{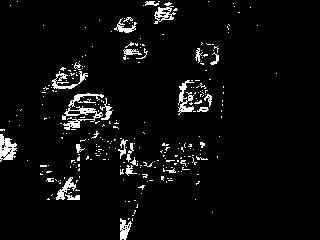}}
\caption{Modeled Background and detected Foregrounds using 10 \textbf{strongest} 
 eigenvectors.}
\label{fig:PCBG}
\end{figure}

Our goal is to prove that the eigenvectors corresponding to the largest eigenvalues of the scatter matrix $S$ are influenced by foreground frames more than background frames. We will use the assumption in Mixture of Gaussian (MoG) for background modeling of Stauffer and Grimson
\cite{Stauffer1999Adaptive}. In their approach each  “pixel process” is a time series of pixel values and these values are modeled by an MoG. In this method it is supposed that:
\begin{quotation}
``$\dots$ the variance of the moving object is expected to remain larger than a background pixel until the moving object stops.''
\end{quotation}

If $\sigma_b$ and $\sigma_f$ show the background and foreground variances, respectively, the above assumption means:
\begin{equation}\label{eq:sblsf}
\sigma_b \ll \sigma_f
\end{equation}

Suppose that  $X$ be  intensity values  of a particular pixel over times.

As a stochastic process, suppose $X = X_b \cup X_f $ is the combination of two stochastic processes $X_b$ and $X_f$ which are those elements that belong to background and foreground respectively. From the elementary statistics we have:
\begin{align}
   \mu_{X_b \cup X_f}   &= \frac{ N_b \mu_b + N_f \mu_f }{N_b + N_f} \label{eq:muX}\\
   \sigma_{X_b\cup X_f}^2 &= { \frac{N_b \sigma_b^2 + N_f \sigma_f^2}{N_b + N_f} + \frac{N_b N_f}{(N_b+N_f)^2}(\mu_b - \mu_f)^2 }
  \end{align}
where $N_b = |X_b|$ and $N_f = |X_f|$.

Based on assumption \eqref{eq:sblsf}, since $\sigma_b \ll \sigma_f$ the foreground pixels variance ($\sigma_f^2$)  has a stronger effect on the total variance.

Let $S_{n-1}=\sum_{k=1}^{n-1} (x_k-m)(x_k-m)^T$ be the scatter matrix, when $X=\{x_1,\dots,x_{n-1}\}$, $x_i\in\mathbb{R}^m$, $i\in\{1,\dots,n-1\}$. We want to show how adding a new instance to $X$ will influence the scatter matrix $S$. We look for changing the main eigenvector of $S$,  based on the new instance, which may or may not  belong to foreground.

According to the Welford algorithm \cite{Welford1962Note}, by adding a new instance $x_n$,  the new mean and scatter matrix can be computed as follows:
\[ \mu_n = \mu_{n-1}+ (x_n - \mu_{n-1})/n \]
\[S_n = S_{n-1} + (x_n - \mu_{n-1})*(x_n - \mu_n)\]

In  multidimensional spaces we have:
\[S_n = S_{n-1} + (x_n - \mu_{n-1})*(x_n - \mu_n)^T\]

With some algebraic manipulations the updated $S_n$ of Welford algorithm can be related only to $\mu_{n-1}$ as follows:
\begin{align}
S_n &= S_{n-1} + \frac{n-1}{n} (x_n - \mu_{n-1})*(x_n - \mu_{n-1})^T \nonumber \\
& = S_{n-1} + yy^T \nonumber
\end{align}
where
\begin{equation}\label{eq:yxn}
  y = \sqrt{\frac{n-1}{n}}(x_n - \mu_{n-1})
\end{equation}

and $yy^T$ is a Hermitian matrix.

\begin{figure}[t]
\centering
\subfigure[Frame No. 16]
{\label{fig:NonPCBG:Frame}
\includegraphics[width=.31\linewidth]{16.jpg}}
\subfigure[Background]
{\label{fig:NonPCBG:BG}
\includegraphics[width=.31\linewidth]{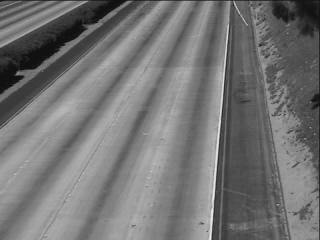}}
\subfigure[Foreground]
{\label{fig:NonPCBG:FG}
\includegraphics[width=.31\linewidth]{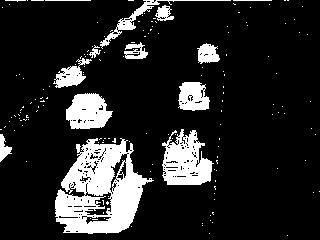}}
\caption{Modeled Background and detected Foregrounds using 10 \textbf{weakest} 
 eigenvectors.}
\label{fig:NonPCBG}
\end{figure}

Next, using the following theorems and lemmas,  we will show the foreground instances have more contribution to the main eigenvector of $S$. Please note that in these theorems, $S_{n-1}, S_{n}$ and $yy^T$ are demonstrated by $A, A'$ and $E$, respectively. 
Also note that we use the terms \textit{dominant eigenvectors}, \textit{the main eigenvectors} and \textit{the strongest eigenvectors} for \textit{the first principal components}. 

As shown in Figure \ref{fig:PC1and2}, suppose the eigenvectors are demonstrated in 2-dimensional space. A new arriving instance may change the dominant eigenvector. We claim  this eigenvector, will be more influenced if the new instance belongs to the foreground. We consider the angle between previous and current main eigenvectors, as a metric that how the new instance affects the main eigenvector. 
First, we will show the difference of the first eigenvectors of $A$ and $A'$ has an upper-bound. Then, we will prove $UB_b \ll UB_f$, in average,  where $UB_b$ is the upper bound of the change in the direction of the main eigenvector, when the new instance is a background instance, and $UB_f$, when it belongs to the foreground.

\begin{theorem}\label{th:normV_Vp}
Let $A \in \mathbb{C}^{n\times n}$ is a Hermitian matrix, $A'=A+yy^*$ is rank one updated of $A$ and $y \in \mathbb{C}^n$ is a column vector.
If $v$ and $v'$ are the normalized eigenvectors of $A$ and $A'$  corresponding to  their largest eigenvalues,  then:
\begin{equation}\label{eq:normV_Vp}
\norm{v-v'}\leq \beta\norm{E}
\end{equation}
where $E=yy^*$ is a small perturbation of $A$, and $\beta$ is a value unrelated to $E$.
\end{theorem}
\begin{proof}
See Theorem \ref{app:th:normV_Vp}
in Appendix.
\end{proof}

The left-hand side of \ref{eq:normV_Vp} is the differences of the main eigenvectors of
 $A$ and $A+E$. It has the upper-bound of  $\beta\norm{E}$, where $E=yy^T$ and $y = \sqrt{\frac{n-1}{n}}(x_n - \mu_{n-1})$. Hence the left-hand side is related to the incoming instance $x_n$. We will show that the expected  value of the eigenvector's variations, is more related to the foreground instances than others.

The expected value of the right-hand side of \ref{eq:normV_Vp} is:
\begin{align}
\mathbb{E}[\beta\norm{E}] =
\beta\mathbb{E}[\norm{E}]
\end{align}

In the next theorem, the relation of this expected value and the incoming instance is shown.
For the sake of simplicity,  we drop the scalar coefficient $\sqrt{\frac{n-1}{n}}$ from $y$, then our perturbation matrix $E$ will be:
\[E = (x_n - \mu_{n-1})(x_n - \mu_{n-1})^T\]
\begin{lemma}\label{lem:Emub}
Let $N_b$, $\mu_b$ be the number and the mean of background instances  until time $n-1$; and $N_f$, $\mu_f$ be the mentioned parameters for the foreground's; and $N_b = N_f \approx (n-1)/2$.
If $E = (x_n - \mu_{n-1})(x_n - \mu_{n-1})^T$ and the new instance $x_n$ belong to Background ($x_n \in B$), then
\[\mathbb{E}[\norm{E}] \leq  \mathbb{E}[ \norm{x_n - \mu_b}^2] + \mathbb{E}[ \norm{x_n - \mu_f}^2] \]
\end{lemma}
\begin{proof}
We know that the spectral norm of a Hermitian matrix is the absolute value of its largest eigenvalue \cite{Stewart1973}. Also, we know that the only non-zero eigenvalue of $E=uv^T$ is $v^Tu$ (See Lemma \ref{app:lem:uvt}
  in Appendix). $E$ is a symmetric matrix, according to
 Lemma 
\ref{app:lem:matNorm}
  we have:
\begin{align}
\norm{E}_2 &= (x_n - \mu_{n-1})^T(x_n - \mu_{n-1})\nonumber
\end{align}
With rewriting $x_n - \mu_{n-1}$ we have:
\begin{align}
x_n - \mu_{n-1}\ &= x_n - \frac{N_b \mu_b+N_f \mu_f}{n-1}\nonumber
\end{align}
Since $N_b = N_f = (n-1)/2$, thus:
\begin{align}
x_n - \mu_{n-1}\ &= x_n - \frac{ \mu_b+ \mu_f}{2}\nonumber \\
& = \frac{(x_n - \mu_b) +  (x_n - \mu_f)}{2}\nonumber
\end{align}
Let $A=(x_n - \mu_b)/2$ and $B=(x_n - \mu_f)/2$.
 Then
\begin{align}
\norm{E}_2 &= (A+B)^T(A+B) \nonumber\\
&= \langle A+B,A+B\rangle\nonumber\\
&=\norm{A}^2+ \langle A,B\rangle+\langle B,A\rangle+\norm{B}^2\nonumber\\
&\leq\norm{A}^2+ 2|\langle A,B\rangle|+\norm{B}^2\nonumber\\
&= \norm{x_n - \mu_b}_2^2 +2 (x_n - \mu_b)^T(x_n - \mu_f)\nonumber\\
&+ \norm{x_n - \mu_f}_2^2\nonumber
\end{align}

Thus:
$\norm{E}_2 \leq \norm{x_n - \mu_b}_2^2 + \norm{x_n - \mu_f}_2^2 + C_1 = \textrm{UB}$
where
$C_1 = 2 (x_n - \mu_b)^T(x_n - \mu_f)$.
The expected value of upper bound $\textrm{UB}$ is given by
\begin{align}
\mathbb{E}[\textrm{UB}] &= \mathbb{E}[ \norm{x_n - \mu_b}_2^2 + \norm{x_n - \mu_f}_2^2 +C_1]\nonumber\\
&=  \mathbb{E}[ \norm{x_n - \mu_b}^2] + \mathbb{E}[ \norm{x_n - \mu_f}^2] + \mathbb{E}[C_1]\label{eq:ExpectedValue}
\end{align}
Since $x_n \in B$, hence $\mathbb{E}[x_n] = \mu_b$, then:
\begin{align}
\mathbb{E}[C_1] & = \mathbb{E}[2 (x_n - \mu_b)^T(x_n - \mu_f)]  \nonumber\\
&= 2\mathbb{E}[x_n^Tx_n - x_n^T\mu_b - x_n^T\mu_f+\mu_b^T\mu_f] \nonumber\\
&= 2( \mathbb{E}[x_n^Tx_n] - \mathbb{E}[x_n^T]\mu_b - \mathbb{E}[x_n^T]\mu_f +\mu_b^T\mu_f\nonumber )\\
&= 2(\mu_b^T\mu_b - \mu_b^T\mu_b - \mu_b^T\mu_f+\mu_b^T\mu_f) = 0 \nonumber
\end{align}
Hence, we have:
\[\mathbb{E}(\norm{E}) \leq  \mathbb{E}[ \norm{x_n - \mu_b}^2] + \mathbb{E}[ \norm{x_n - \mu_f}^2] \]
\end{proof}

For the sake of simplicity, temporarily, we drop $n$ from $x_n$, and $x$ denotes the incoming frame, which may be belong to background ($B$) or foreground ($F$).
 
\begin{lemma}\label{lem:sigmab}
Suppose that $x = [x^1,\dots,x^j,\dots,x^m]^T\in B$ 
 indicate an instance of background frames ($B$) over time.
If $(\sigma^j_b)^2$ is the variance of the $j^{th}$ component of ${x}$, then
$\mathbb{E}[ \norm{x - \mu_b}^2] =  \sum_{j=1}^m(\sigma_b^j)^2$.
\end{lemma}
\begin{proof}
Assume that
$x,\mu_b  \in \mathbb{R}^m$, then:
\begin{align}
 \norm{x - \mu_b}^2 &= \sum_{j=1}^m (x^j - \mu_b^j)^2   \textrm{  follows that}\nonumber\\
\mathbb{E}[ \norm{x - \mu_b}^2] &= \mathbb{E}[ \sum_{j=1}^m (x^j - \mu_b^j)^2]\nonumber\\
&=   \sum_{j=1}^m\mathbb{E}[ (x^j - \mu_b^j)^2] =  \sum_{j=1}^m(\sigma_b^j)^2  \tag{$x \in B$}
\end{align}
\end{proof}

If $x\in F$ (Indicating a foreground instance), with the same induction of the previous Lemma, we will have:
\begin{align}
\mathbb{E}[ \norm{x - \mu_f}^2] =   \sum_{j=1}^m(\sigma_f^j)^2
\end{align}

According to the assumption \ref{eq:sblsf}, we have:
$$(\sigma_b^j )^2\ll (\sigma_f^j)^2, \quad \forall j$$
If the above summations $ \sum_{j=1}^m(\sigma_b^j)^2 , \sum_{j=1}^m(\sigma_f^j)^2 $ are demonstrated by $\Sigma_b^2, \Sigma_f^2$, we have:
\begin{equation} \label{eq:SbllSf}
\Sigma_b^2 \ll \Sigma_f^2
\end{equation}

Here we investigated the expected value of the upper bound  $\textrm{UB}$, where $x$ belongs to background or foreground.  At first, suppose $x\in B$, hence we have:
\[\mathbb{E}[x] = \mu_b\]

\begin{proposition}\label{prop:EsigmabC}
With notation as above,
if $x \in B$, then
\[\mathbb{E}[\norm{E}] \leq  \Sigma^2_b+C \]
where $C$ is a constant, unrelated to $x$.
\end{proposition}
\begin{proof}
\begin{align}
\mathbb{E}(\norm{E}) &\leq  \mathbb{E}[ \norm{x - \mu_b}^2] + \mathbb{E}[ \norm{x - \mu_f}^2] \tag{According to Lemma \ref{lem:Emub}}\\
& = \Sigma^2_b + \mathbb{E}[ \norm{x - \mu_f}^2] \tag{According to Lemma \ref{lem:sigmab}}
\end{align}
In fact,
\begin{align}
\mathbb{E}[ \norm{x - \mu_f}^2] &= \mathbb{E}[x^Tx - 2x^T\mu_f+\mu_f^T\mu_f] \nonumber\\
&= \mathbb{E}[x^Tx] - 2\mathbb{E}[x^T]\mu_f+\mu_f^2\nonumber\\
&= \mu_b^T\mu_b - 2\mu_b^T\mu_f+\mu_f^T\mu_f = C\nonumber
\end{align}
Thus,
\[\mathbb{E}[\norm{E}] \leq \Sigma_b^2 + C\]
\end{proof}

\begin{remark}\label{remark:SigmaF}
Lemma \ref{lem:Emub}, \ref{lem:sigmab} and Proposition \ref{prop:EsigmabC} were proved
 when $x \in B$. With similar induction it is clear that these are true when $x \in F$ (the new instance belongs to the foreground). Hence these Lemmas and Proposition are true for every instance (whether belongs to either background or foreground). Note that for foreground instances $\Sigma^2_b$ should be replaced by $\Sigma^2_f$, i.e.:
 \[\mathbb{E}[\norm{E}] \leq \Sigma_f^2 + C\]
\end{remark}

\begin{theorem}\label{th:angleV_Vp}
Let $A \in \mathbb{C}^{n\times n}$ be a Hermitian matrix, $A'=A+yy^*$ be a rank one updated of $A$ and $y \in \mathbb{C}^n$ be column vector constructed by the arrived instance $x_n$ (eq. \ref{eq:yxn}).
Suppose $v$ and $v'$ be the normalized eigenvectors of $A$ and $A'$  corresponding to  their largest eigenvalues, and $\theta_b$($\theta_f$) be the angle between $v$ and $v'$, when the new instance belongs to the background(foreground).
If $\Sigma_b^2 \ll \Sigma_f^2$ then
$$ \mathbb{E}[\theta_b] \ll \mathbb{E}[\theta_f]$$
\end{theorem}

\begin{proof}
According to Theorem \ref{th:normV_Vp}, the upper-bound of $\norm{v-v'}$ is $\beta\norm{E}$ and according to Proposition \ref{prop:EsigmabC} and Remark \ref{remark:SigmaF}, the expected value of the upper bound is:
\begin{align}\label{eq:E_V_Vp}
 \mathbb{E}[\theta_b]= \mathbb{E}[\norm{v-v'}] \leq \beta(\Sigma^2_b + C) \quad \text{if } x_n\in B \\
  \mathbb{E}[\theta_f]=\mathbb{E}[\norm{v-v'}] \leq \beta(\Sigma^2_f + C) \quad \text{if } x_n\in F
\end{align}

For a specific signal $X=\{x_1,\dots,x_{n-1}\}$ the values $\beta,C$ are unrelated to incoming instance $x_n$; 
since $\Sigma^2_b \ll \Sigma^2_f$, we have $ \mathbb{E}[\theta_b] \ll \mathbb{E}[\theta_f]$.
\end{proof}

\begin{corollary}\label{th:main}
Let $X=\{x_1,\dots,x_{n-1}\}$
be the previous instances of the signal, and $x_n\in B\cup F$ is the last incoming instance; and $B, F$ denote the Background or the Foreground sets.
Suppose that the variances of background instances be less than foregrounds.
Then it is expected that the first principal component of $X$ to
be more affected when $x_n\in F$, than when $x_n\in B$.
\end{corollary}

\begin{proof}
The first principal component of $X$ is the normalized eigenvector $v$ corresponding to the largest eigenvalue of scatter matrix $A=S_{n-1}=\sum_{k=1}^{n-1} (x_k-\mu)(x_k-\mu)^T$.
When a  new instance $x_n$ arrives, the new eigenvector $v'$ is about the same for  the new scatter matrix $A'=S_n$. 
$A$ is a Hermitian matrix and $A'$ is a rank one updated of $A$.
So, if $\theta_b$ be the angle between $v$ and $v'$ when $x_n\in B$ and $\theta_f$ when $x_n\in F$, then
according to Theorem \ref{th:angleV_Vp}, the proof is straightforward ($ \mathbb{E}[\theta_b] \ll \mathbb{E}[\theta_f]$).
\end{proof}

\section{Experimental Results} \label{Sec:Exp}

In this section some further experiments are demonstrated to support the discussions in sections \ref{sec:review} and  \ref{Sec:Math}. 
First, the result of  corollary \ref{th:main} is verified in Subsection \ref{sec:change_main_eigenvector}.
In Subsection \ref{sec:RecError} we discuss the performance of strongest and weakest eigenvectors in image reconstruction vs. background modeling and show the results for two video data. 
In Subsection \ref{sec:subspace} we show the subspace spanned by the strongest and the weakest eigenvectors for 3 different videos, and we also show the weakest eigenvectors can well represent the background subspace. Lastly, we present an experiment in Subsection \ref{sec:various_sizes} and discuss about the effect of the size of foreground objects in the performance of background models created by the strongest and the weakest eigenvectors. 

\subsection{Changes in the main eigenvector}\label{sec:change_main_eigenvector}

\begin{figure}
\centering
\includegraphics[width=.9\linewidth]{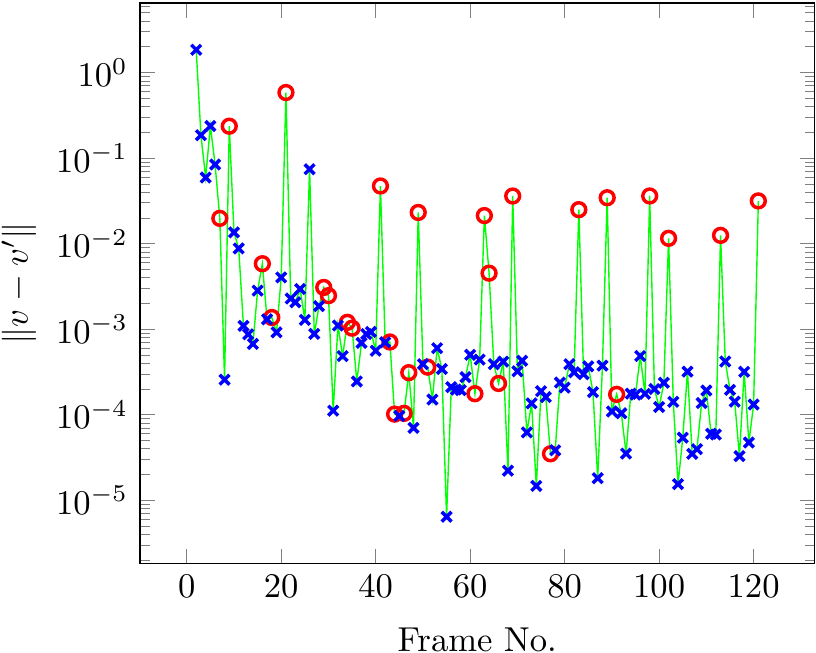}
\caption{
Changes in the main eigenvector (i.e. the eigenvector belonging to the largest eigenvalue)  for every new frame added by calculating the norm of the difference between the old vector $v$ and the new vector $v'$. The blue crosses demonstrate background frames and the red circles are for foreground frames. This shows that foreground frames are more influential than background frames in changing the main eigenvector.}
\label{fig:Snorms}
\end{figure}

In Figure \ref{fig:Snorms} we calculated and showed the norm differences of the main eigenvectors for every new frame added ($\norm{v-v'}$).
As can be seen in this figure, the main eigenvector is more influenced by foreground frames than background ones. This evidence supports Theorems \ref{th:main} and \ref{th:angleV_Vp}. In average, eigenvector's variations due to frames with foreground objects are greater than those related to backgrounds. Frames and their orders are as same as figure \ref{fig:PC1and2}.

\subsection{Image Reconstruction vs. Background Modeling}\label{sec:RecError}

The strongest eigenvectors are typically used in data compression and visualization 
\cite[Chap 14.]{HastieEtAl2008}, \cite[Chap 12.]{Bishop2006}. 

 Image reconstruction by the strongest eigenvectors is a form of data compression, which is used in eigenfaces and eigenbackgrounds. Here we compare the effect of using the strongest and the weakest eigenvectors for image reconstruction and compute both the reconstruction errors and background modeling errors.
 This has been shown in the top rows of figure \ref{fig:RMSE_FGs}. In this figure, the data from the Highway video shown in figure \ref{fig:SampleVideo} is used for the experiment. We first built a \textit{Base-Model (BM)} using selected number of eigenvectors (i.e. 1\textsuperscript{st}, 1\textsuperscript{st} to 7\textsuperscript{th}, 1\textsuperscript{st} to 30\textsuperscript{th}, 24\textsuperscript{th} to 30\textsuperscript{th}, and 30\textsuperscript{th} only as shown in figure \ref{fig:RMSE_FGs} horizontal axis), and then reconstruct all image frames in the video using this BM  %
 \footnote{BM is a user-friendly name of $\Phi_M$, mentioned in page 1.}

To segment the foreground objects from the background in a new image $I$, it is projected into the subspace $BM$: $I_{proj} = BM^T\times(I-\mu)$ and the reconstructed image $I_{rec}$ is considered as the estimated background ($\widehat{BG}$):
$\widehat{BG} \equiv I_{rec} = BM\times I_{proj}+\mu$. The difference between $I$ and $I_{rec}$ is calculated and foreground pixels are classified by a threshold \cite{Hughes2013}.

 The Root Mean Square Error (RMSE) of this image reconstruction is shown in the first row of this figure, where each line belongs to one frame (total of 30 frames). In this figure, we have highlighted the frames that have almost no foreground objects in blue color. As can be seen in this figure, the reconstruction error decreases by increasing the number of eigenvectors used in the BM. Also, we can see the weakest eigenvector has poor performance in image reconstruction.

Now let's study how the BM developed with various eigenvectors can be used as the background model of this video frame. For this, we first manually built a \textit{Ground-Truth (GT)} background model by averaging the video frames that do not contain any foreground objects (i.e. frames highlighted by blue lines). Then for each frame in the video, we calculate the RMSE between the reconstructed image (generated by applying each frame into the BM). The RMSE results are shown in the second row of figure \ref{fig:RMSE_FGs}. As can be seen in this figure, the BMs created by the weakest eigenvectors have consistently low RMSE for all video frames no matter if the frame contains any foreground objects or not. On the other hand, the BM created by the strongest eigenvector is very sensitive to the input image. It has good performance (low RMSE) if the image has no foreground object and poor performance when the image contains large foreground object. As shown in this figure, the video frame with a very large foreground object has made the worst performance while the RMSE is low for the frame with a very small foreground object.

\begin{figure}[t]
\centering
\includegraphics[width=1\columnwidth]{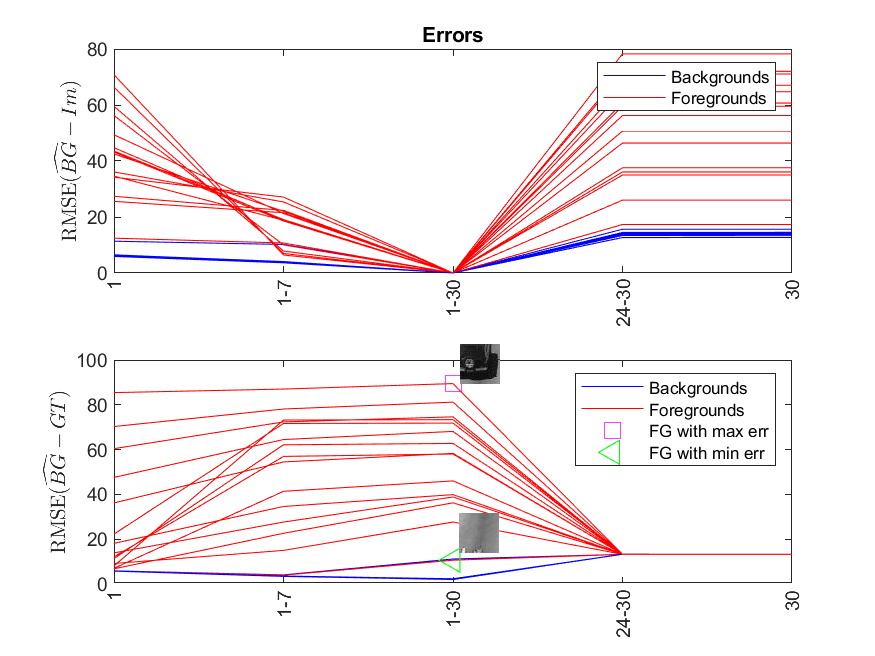}
\caption{
Results of the Based-Model (BM) development on \textit{Highway} video. The reconstructed image by selected eigenvectors is the estimated background ($\widehat{BG} \equiv I_{rec}$). First row: Reconstruction Error, RMSE of 
$\widehat{BG}-I$.
Second row: RMSE of 
$\widehat{BG}-GT$ (GT: Ground Truth background model).
Each  frame as an image is reconstructed on eigenspace on various combinations of eigenvectors; 5 subsets are reported in this figure: \{1\}, \{1,\dots,7\}, \{All 30 eigenvectors\}, \{last 7 eigenvectors\} and the last eigenvector (\{30\}) . The result for an image with foreground objects with minimum average error is highlighted by a green triangle, and the one with highest error is highlighted by a purple square. As can be seen from the first row, we can achieve zero reconstruction error by selecting all eigenvectors.  On the other hand, although the weakest eigenvectors have large reconstruction error, but they have consistently low error in background modeling in compared with the ground truth background (second row).
}
\label{fig:RMSE_FGs}
\end{figure}

Figure \ref{fig:7evs} shows the results of foreground detection using these BMs on four frames from the video (frames 5, 9, 15, and 17). As can be seen in this figure, the reconstructed image ($\widehat{BG} \equiv I_{rec}$) generated by the top 7 eigenvectors are very close to the actual frames, while the results from the top 7 weakest ones are good for background modeling, and hence they have good performance in foreground detection (\textit{FG}) too.   

We conducted similar experiment on an indoor video data (\textit{ShoppingMall} video).
 The results are shown in figure \ref{fig:ShoppingMall:RMSE_BGs}. Here we first did the same experiment as in figure \ref{fig:RMSE_FGs} on the first 30 frames of the video (see the two diagrams in figure \ref{fig:ShoppingMall:RMSE_BGs:a}) and then we used the 30 eigenvectors of these 30 frames to evaluate
 200 frames of the video (the two diagrams of figure \ref{fig:ShoppingMall:RMSE_BGs:b}). Again, we can see the weakest eigenvector has consistently lower RMSE in background modeling compared to those created by strongest eigenvectors.

\begin{figure}[t]
\centering
\includegraphics[width=0.8\linewidth]{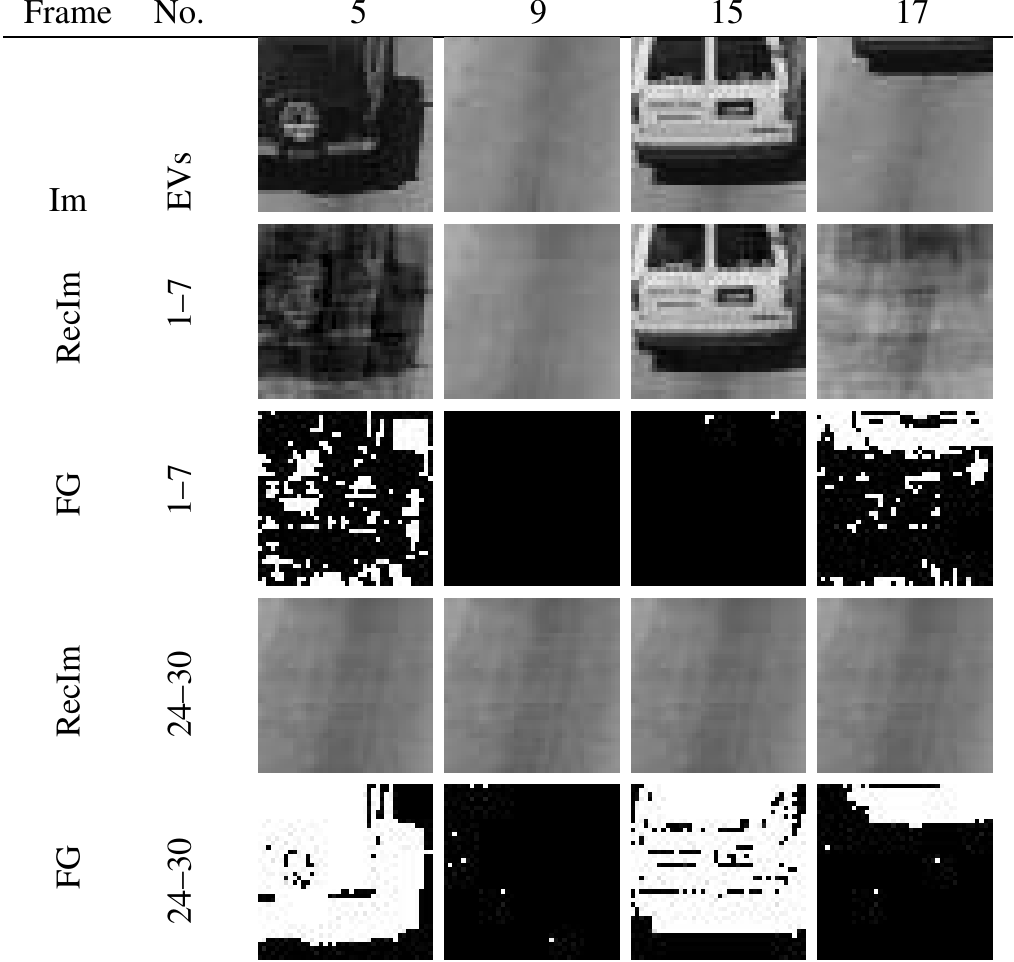}
\caption{The results of foreground detection for background models developed using i) Top 7 eigenvectors (2\textsuperscript{nd}  and 3\textsuperscript{rd}  rows), and ii) last 7 eigenvectors (4\textsuperscript{th}  and 5\textsuperscript{th}  rows) for the experiment of figure \ref{fig:RMSE_FGs}. Clearly, the weak eigenvectors have better background modeling.}
\label{fig:7evs}
\end{figure}

\begin{figure}
\centering
\subfigure[] 
{\label{fig:ShoppingMall:RMSE_BGs:a}
\includegraphics[width=1\columnwidth]{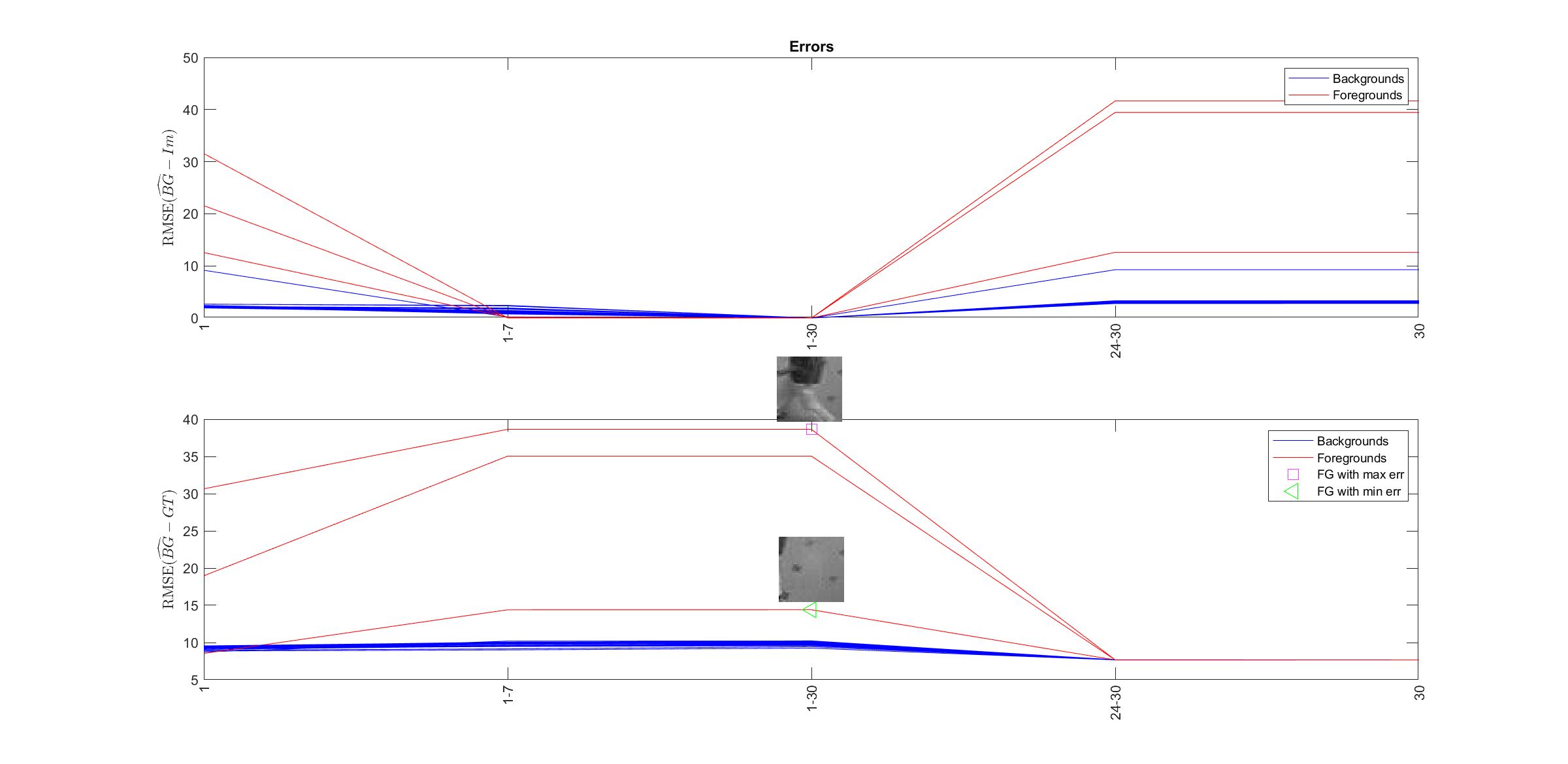}}
\\
\subfigure[] 
{\label{fig:ShoppingMall:RMSE_BGs:b}
\includegraphics[width=1\columnwidth]{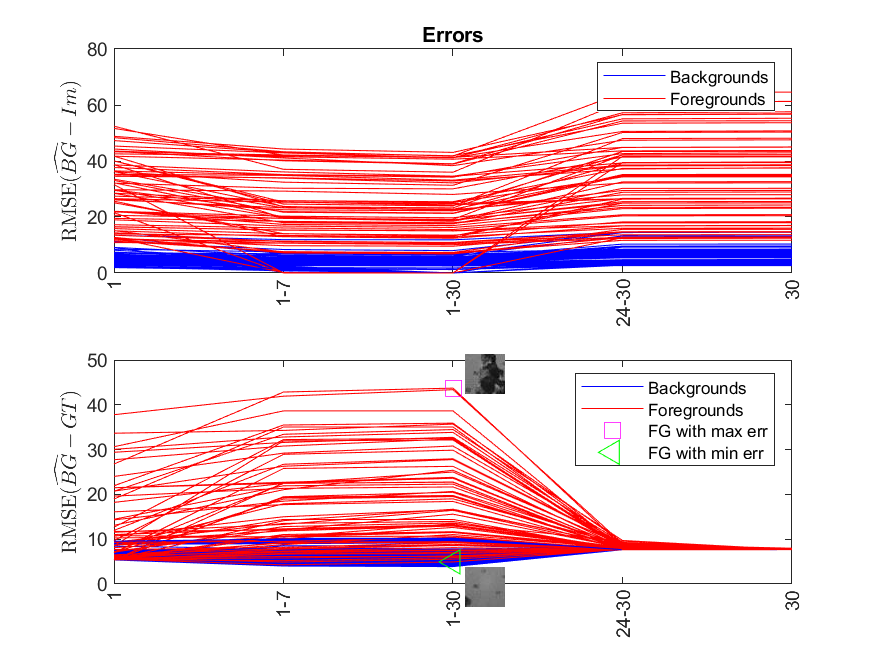}}
\caption{
Same experiment as figure \ref{fig:RMSE_FGs} but on \textit{ShoppingMall} video. (a) Results using the BM created with the first 30 frames in the video. (b) Results of the BM created in (a) on 200 frames of the video. 
}

\label{fig:ShoppingMall:RMSE_BGs}
\end{figure}

\subsection{Spanned subspaces}\label{sec:subspace}

As another experiment, for the same block shown in Figure \ref{fig:SampleVideo}, we computed all 121 eigenvectors (121 frames resulted in 121 eigenvectors). Then for every frame, the projected image on the subspace of the first two  eigenvectors is demonstrated in Figure \ref{fig:Frames:Highway}.
This figure shows how image frames are distributed in the 2D space defined by
two consequence eigenvectors.
As before, the red circles demonstrate foreground frames and blue crosses show background frames.
Subfigures (a)--(e) show the subspaces corresponding to the following eigenvectors' pairs:
(1,2),
(25, 26),  (49, 50), (73, 74)
and (97, 98), respectively.


In figure \ref{fig:Frames:Highway}(a), the left panel shows the first and second eigenvectors. As can be seen, the blue crosses that belong to background frames are mapped very closely while the red circles, demonstrated the foreground frames, are  distributed in this 2D space. This shows this subspace is indeed suitable for foreground objects and not for background regions.  For better visualization, the right panel shows some  uniformly selected frames from the left panel points. If we divide the whole space into a 5$\times$5 hypothetical grid, these 25 images are the closest projected images to the vertices of this grid. The corresponding points are marked by a green plus sign in the left panel (behind red and blue markers).

Similarly, we created a 2D space defined by some other eigenvectors 
and again mapped all frames into this 2D space as explained above. The result is shown in Figure \ref{fig:Frames:Highway}(b)--(e). Figure (e) shows less significant components (97 and 98). In contrast to Figure \ref{fig:Frames:Highway}(a), in the 2D space defined by these non-significant eigenvectors, all foreground frames are mapped very closely to each other while background frames are widely distributed.
These figures show most significant eigenvectors are more suitable to model and analysis of different foreground objects while the space defined by the least significant eigenvectors is more proper for modeling the background.

Figures \ref{fig:Frames:ShoppingMall} and \ref{fig:Frames:Traffic} show the aforementioned experiments on two other videos: \textit{ShoppingMall} and \textit{Traffic}, where both of them have more crowded scenes than \textit{Highway}. As the previous results in figure  \ref{fig:Frames:Highway}, the most important eigenvectors, illustrate the foreground space, and by
increasing the  component numbers -- or least significant eigenvectors -- the background frames (blue markers) are more distributed. In the last two videos, even in the least significant components, red markers (or foreground frames) are also spread out over the whole space. The reason is that these videos are taken from a very crowded place, and almost in all frames, a moving object was present in every block.

As can be seen,  in all three figures, the subspace spanned by the weakest eigenvectors  -- subfigures (a) -- the background instances (blue crosses) are compacted and not spread well. By increasing the principal component number, background instances are spread and  foreground ones are compressed.


\begin{figure}
\centering
\includegraphics{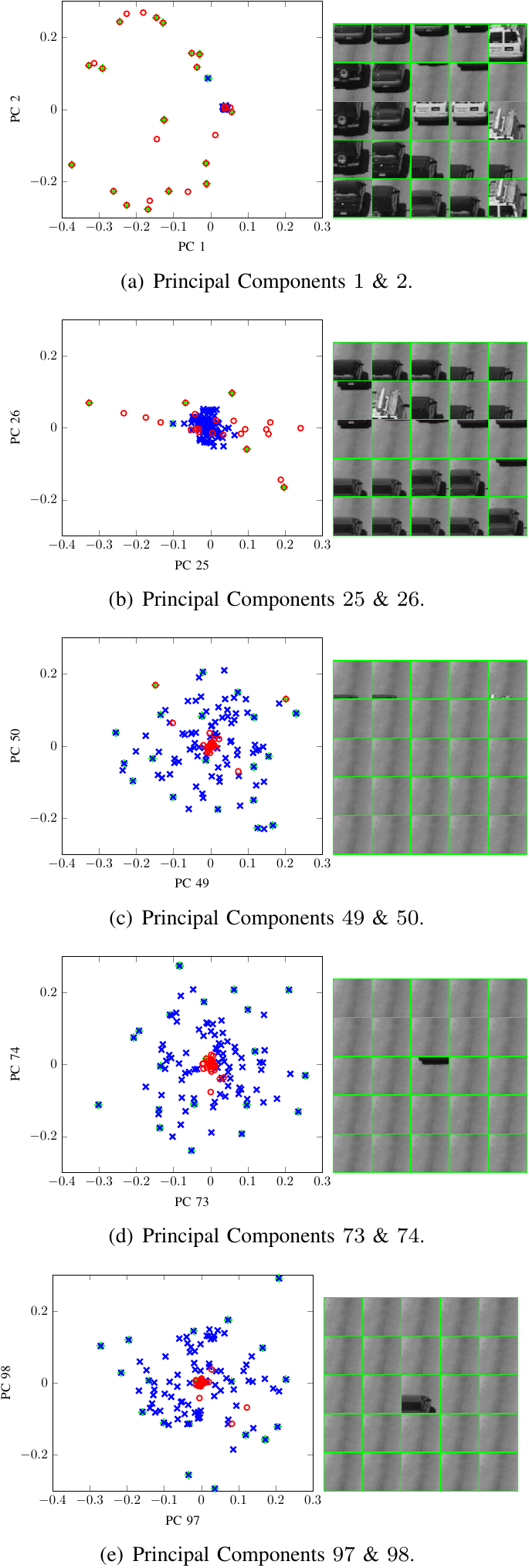}
\caption{Some demonstration for the effect of different principal components subspaces for video: \textit{\textbf{Highway}}.  Left panel: the principal components' subspace of the  video. The blue crosses/red circles, show background/foreground images. The green plus points are the closest projected images to the vertices of a unified grid, defined by the marginal of the principal components. Right panel: The images correspond to the green plus points. Foreground images are well distributed in subspaces related to the strongest eigenvectors (first eigenvectors); in contrast background frames are well propagated in subspaces corresponding to the weakest eigenvectors. }
\label{fig:Frames:Highway}
\end{figure}

\begin{figure}
\centering
\includegraphics{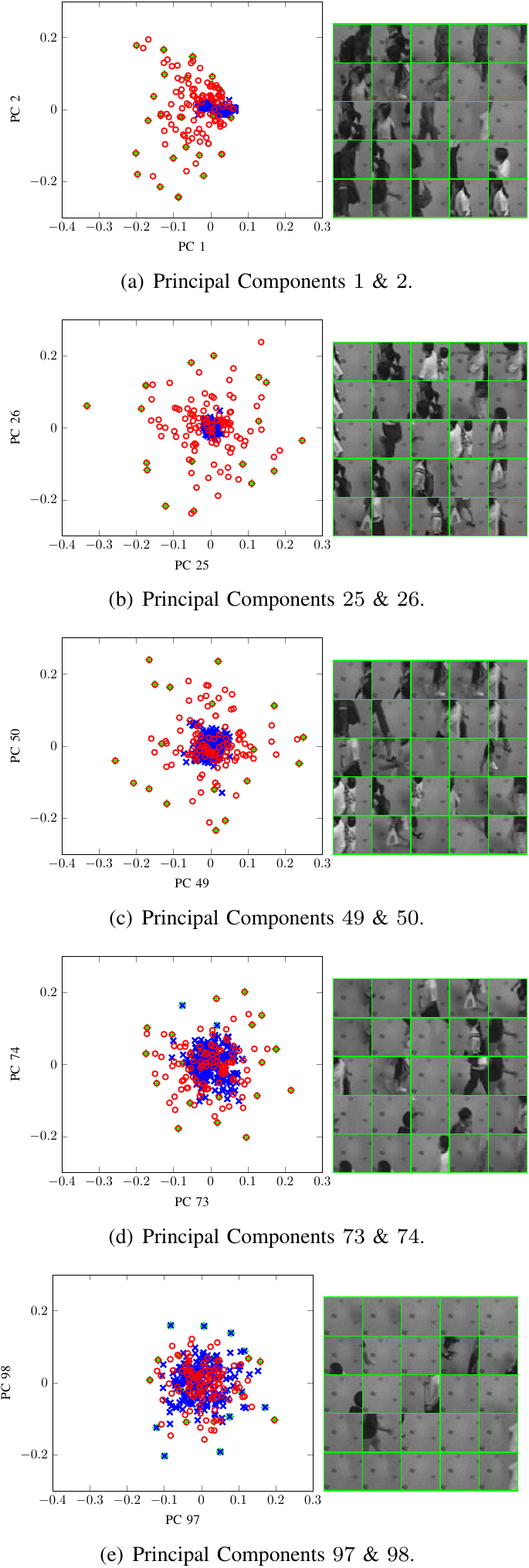}
\caption{Some demonstration for the effect of different principal components subspaces for video: \textit{\textbf{ShoppingMall}}.  Left panel: the principal components' subspace of the  video. The blue crosses/red circles, show background/foreground images. The green plus points are the closest projected images to the vertices of a unified grid, defined by the marginal of the principal components. Right panel: The images correspond to the green plus points. Background frames are getting propagated as we move from the strongest eigenvector subspace towards the subspaces corresponding to the weakest eigenvectors. }
\label{fig:Frames:ShoppingMall}
\end{figure}

\begin{figure}
\centering
\includegraphics{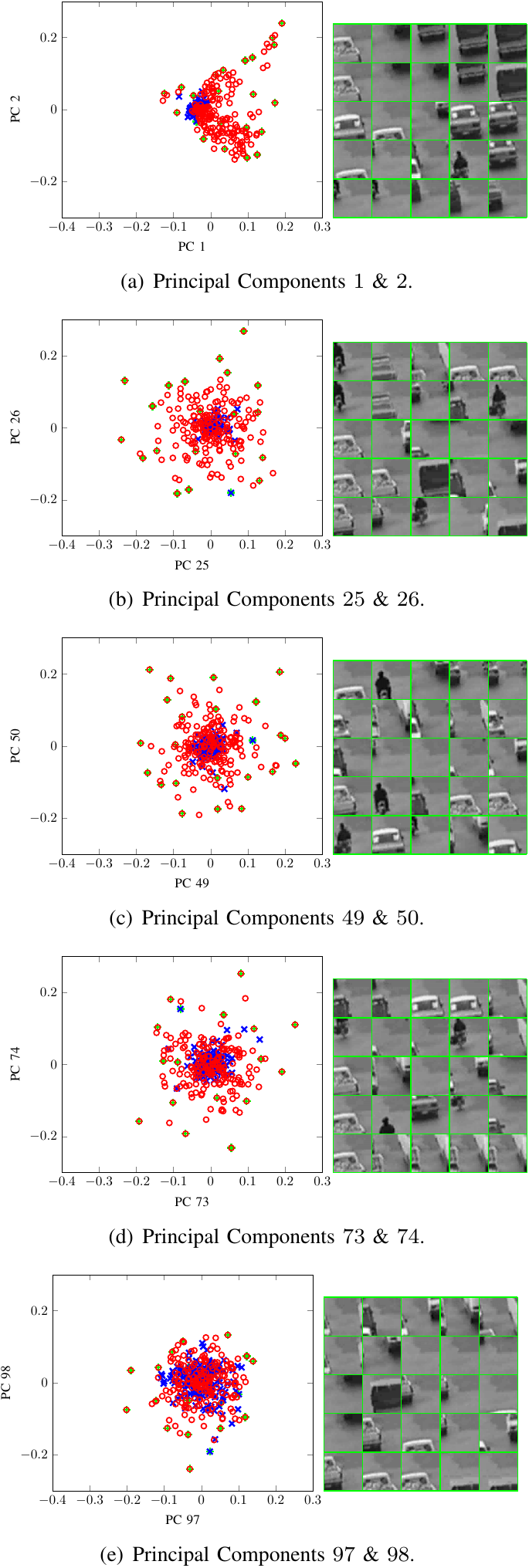}
\caption{Some demonstration for the effect of different principal components subspaces for video: \textit{\textbf{Traffic}}.  Left panel: the principal components' subspace of the  video. The blue crosses/red circles, show background/foreground images. The green plus points are the closest projected images to the vertices of a unified grid, defined by the marginal of the principal components. Right panel: The images correspond to the green plus points. Again, background frames are getting propagated as we move from the strongest eigenvector subspace towards the subspaces corresponding to the weakest eigenvectors. }
\label{fig:Frames:Traffic}
\end{figure}

\subsection{Effect of the Foreground Object's Size}\label{sec:various_sizes}

 In this section we study the size of a foreground object on the performance of the background model generated by the Eigenbackground algorithm, as well as the results of our proposed method on using the weakest eigenvectors.
Here we use a window of \textit{ShoppingMall} video used in figure \ref{fig:Frames:ShoppingMall} with various sizes: $32\times32$, $64\times64$, $128\times128$ and $256\times256$ as shown in figure \ref{fig:ShoppingMall:32_64_128_256}. The results are demonstrated in figures \ref{fig:ShoppingMall:32} -- \ref{fig:ShoppingMall:256}. In these figures, we show the results of foreground detection without  any filtering or post-processing. It is clear from these figures that  i) the performance of the Eigenbackground method (that uses the strongest eigenvectors) is worsen by increasing the size of the foreground object, and even for the case of small size foreground object, still there is some ghosting effect in the generated background image,  and ii) our proposed method on using the weakest eigenvectors performs well and it is independent to the size of the foreground object.

\begin{figure}[t]
\centering
\includegraphics[width=.95\linewidth]{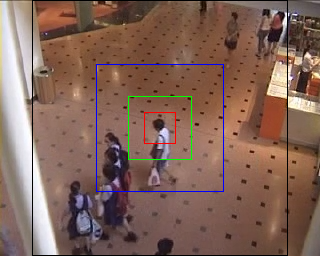}
\caption{  A test scenario to study the effect of foreground object's size. The background model to be created on various block size:  $32\times32, 64\times64, 128\times128$ and $256\times256$.
\label{fig:ShoppingMall:32_64_128_256} }

\end{figure}


\begin{figure}[t]
\centering
\includegraphics[width=0.67\linewidth]{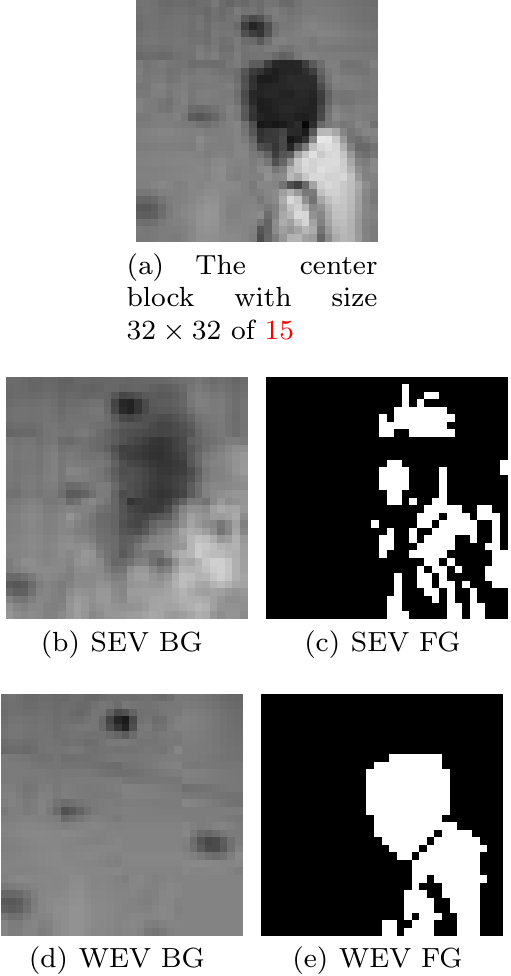}
\caption{Background model and detected Foreground using  10 Strongest Eigenvectors (SEV) and 10 Weakest Eigenvectors (WEV)  for the center block of figure  \ref{fig:ShoppingMall:32_64_128_256} with block size = 32.}
\label{fig:ShoppingMall:32}
\end{figure}

\begin{figure}[t]
\centering
\includegraphics[width=0.67\linewidth]{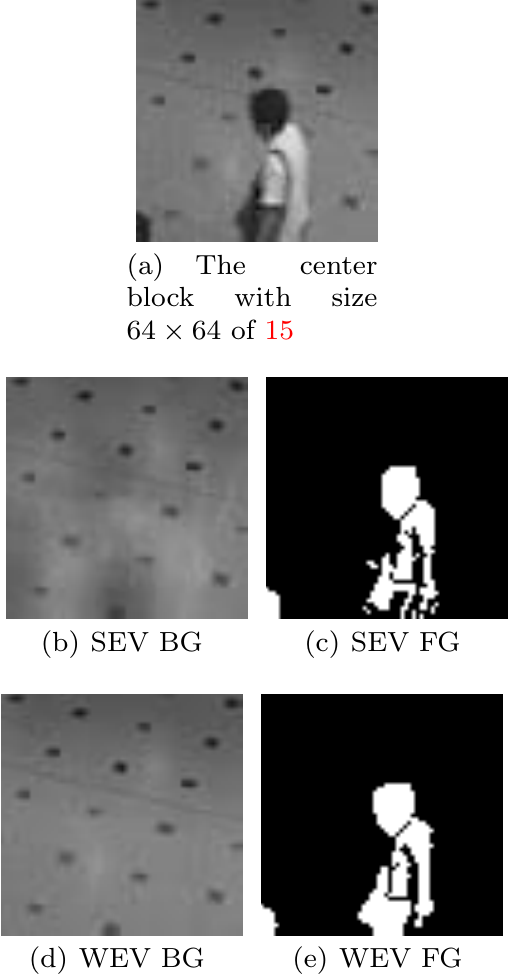}
\caption{Background model and detected Foreground using  10 Strongest Eigenvectors (SEV) and 10 Weakest Eigenvectors (WEV)  for the center block of figure  \ref{fig:ShoppingMall:32_64_128_256} with block size = 64.}
\label{fig:ShoppingMall:64}
\end{figure}

\begin{figure}[t]
\centering
\includegraphics[width=0.67\linewidth]{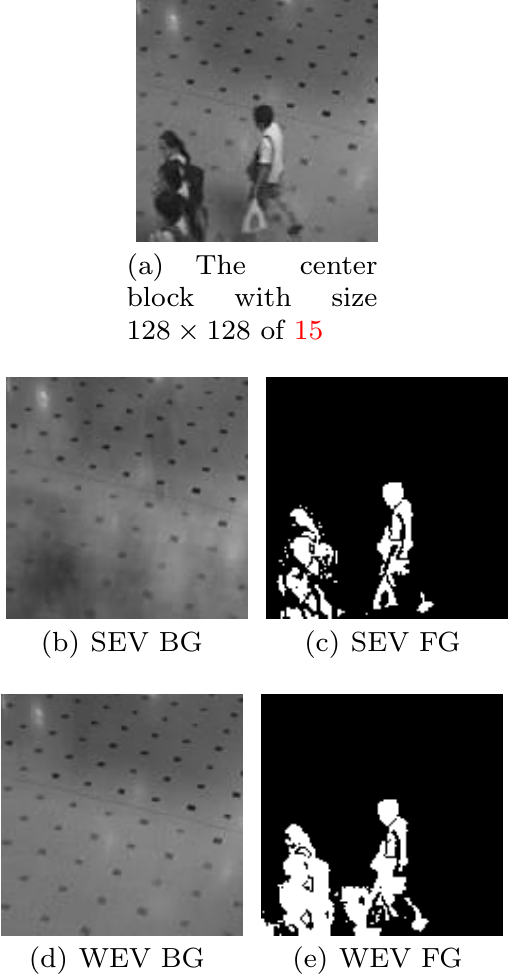}
\caption{Background model and detected Foreground using  10 Strongest Eigenvectors (SEV) and 10 Weakest Eigenvectors (WEV)  for the center block of figure  \ref{fig:ShoppingMall:32_64_128_256} with block size = 128.}
\label{fig:ShoppingMall:128}
\end{figure}

\begin{figure}[t]
\centering
\includegraphics[width=0.67\linewidth]{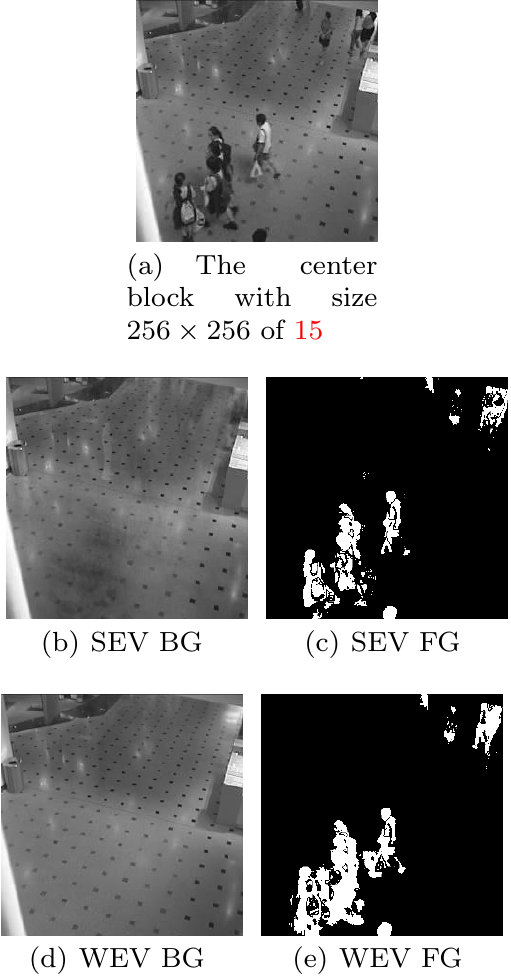}
\caption{Background model and detected Foreground using  10 Strongest Eigenvectors (SEV) and 10 Weakest Eigenvectors (WEV)  for the center block of figure  \ref{fig:ShoppingMall:32_64_128_256} with block size = 256.}
\label{fig:ShoppingMall:256}
\end{figure}

\section{Conclusion}\label{Sec:Conc}
Eigenbackground \cite{ Oliver2000PCA} is a well-known method for background modeling, 
where the strongest (i.e. most significant) eigenvectors are selected for background modeling. 
In this paper,  this method was investigated and proved that the subspace produced by the strongest eigenvectors is affected by both foreground objects and background and hence a background model created in this way suffers from ghost effects. 
Theoretical aspects of this investigation were also confirmed by experimental results.
This finding explains
 various artifacts reported by researchers on using the Eigenbackground method.

On the other hand, this study shows the weakest  ( i.e. least significant) eigenvectors, which are generally considered less important in computer vision and image processing applications (as they are considered to be highly affected by noise and other distortions in the input videos) can indeed play useful role in background modeling. This can be an alternative to the original Eigenbackground model to reduce the dependency of the background model to the foreground objects in the scene. 

To the best of our knowledge, this paper is the first that propose to use the weakest eigenvector for a scientific application, as traditionally,  this information is thrown away and treated as ``garbage data''.  This can open a new door for researchers in computer vision, pattern analysis, and machine learning to reconsider and study the weakest eigenvector for other potential applications  too.

\bibliographystyle{ieeetr}
\bibliography{references,OliverCitations}

\appendix
\section{Appendix}
\begin{theorem}\label{app:th:normV_Vp}
Suppose that $A \in \mathbb{C}^{n\times n}$ is a Hermitian matrix, $A'=A+yy^*$ is rank one updated of $A$ and $y \in \mathbb{C}^n$ is column vector.
If $v$ and $v'$ are the normalized eigenvectors of $A$ and $A'$  corresponding to  their largest eigenvalues,  then:
\begin{equation}\label{app:eq:normV_Vp}
\norm{v-v'}\leq \beta\norm{E}
\end{equation}
where $E=yy^*$ is a small perturbation of $A$ and $\beta$ is a value unrelated to $E$.
\end{theorem}

\begin{proof}
Suppose that $\lambda$ and $\lambda'$ are the largest eigenvalues of $A$ and $A'$. By definition of eigenvector we have
\begin{align}
Av &= \lambda v \nonumber\\
A'v' &= \lambda' v' \nonumber
\end{align}
According to Proposition \ref{app:prop:lambdaMax} there is an $\epsilon>0$ such that $\lambda+\epsilon = \lambda'$. Hence:
\begin{align}
 (A+E)v' - Av  &= \lambda' v' - \lambda v \nonumber\\
 & =  (\lambda + \epsilon)v' - \lambda v  \nonumber\\
\end{align}
Thus:
\begin{align}
 A(v' -v) + Ev'  &= \lambda v'  + \epsilon v' - \lambda v \nonumber\\
\Rightarrow \nonumber \\
 (A-\lambda I)(v'-v) & = (\epsilon I - E)v' \nonumber\\
\Rightarrow \nonumber \\
 v'-v &= {(A-\lambda I)^{-1} (\epsilon I - E)v' } \nonumber\\
\Rightarrow \nonumber \\
 \norm{v'-v} &= \norm{(A-\lambda I)^{-1} (\epsilon I - E)v' } \nonumber\\
&\leq \underbrace{\norm{(A-\lambda I)^{-1}}}_{\alpha}\norm{(\epsilon I - E)}\underbrace{\norm{v'}}_1 \nonumber\\
&= \alpha \norm{(\epsilon I - E)} \nonumber\\
&\leq \alpha (\norm{\epsilon I} + \alpha\norm{-E}) \nonumber \tag{Triangle Inequality}\\
&= \alpha( \epsilon + \norm{E})\nonumber\\
&\leq \alpha(\norm{E}+ \norm{E})\nonumber \tag{According to Proposition \ref{app:prop:lambdaMax}: $\epsilon\leq \norm{E}$}\\
&= \beta\norm{E}
\end{align}
where $\beta = 2\alpha$
\end{proof}

\begin{lemma}\label{app:lem:lambdaMax}
For a given Hermitian matrix $A \in \mathbb{C}^{n\times n}$ and a column vector $y \in \mathbb{C}^n$, we have:
\[\lambda_{max}(A) \leq \lambda_{max}(A+yy^*) \leq \lambda_{max}(A) + \norm{y}^2 \]
\end{lemma}
\begin{proof}
See \cite{IpsenN09} .
\end{proof}

\begin{lemma}\label{app:lem:matNorm}
The norm of a symmetric matrix is maximum absolute value of its eigenvalue.
\end{lemma}
\begin{proof}
We have
$$\|A\|_2=\max_{\|x\|=1}\|Ax\|$$
where $\|\cdot\|$ denotes the ordinary Euclidean norm. This is a constrained optimization problem with Lagrange function:
$$L(x,\lambda)=\|Ax\|^2-\lambda(\|x\|^2-1)=x^TA^2x-\lambda(x^Tx-1)$$
Taking squares  makes the following step easier.
Taking derivative with respect to $x$ and equating it to zero we get
$$A^2x-\lambda x=0$$
the solution for this problem is the eigenvector of $A^2$. Since $A^2$ is symmetric, all its eigenvalues are real. So $x^TA^2x$  will achieve maximum on set $\|x\|^2=1$ with maximal eigenvalue of $A^2$. Now since $A$ is symmetric it admits representation
$$A=Q\Lambda Q^T$$
with $Q$ the orthogonal matrix and $\Lambda$ diagonal with eigenvalues in diagonals. For $A^2$ we get
$$A^2=Q\Lambda^2 Q^T$$
so the eigenvalues of $A^2$ are squares of eigenvalues of $A$. The norm $\|A\|_2$ is the square root taken from maximum $x^TA^2x$ on $x^Tx=1$, which will be the square root of maximal eigenvalue of $A^2$ which is the maximal absolute eigenvalue of $A$ \cite{mathSE}.
\end{proof}

\begin{lemma}\label{app:lem:uvt}
Suppose $A=uv^T$ where $u$ and $v$ are non-zero column vectors in ${\mathbb R}^n$, $n\geq 3$.
Then  $\lambda=0$ and  $\lambda=v^Tu$ are the only eigenvalues of $A$.
\end{lemma}
\begin{proof}
$\lambda=0$ is an eigenvalue of $A$ since $A$ is not of full rank. $\lambda=v^Tu$ is also an eigenvalue of $A$ since $$Au = (uv^T)u=u(v^Tu)=(v^Tu)u.$$
We assume $v\ne 0$. The orthogonal complement of the linear subspace generated by $v$ (i.e. the set of all vectors orthogonal to $v$) is therefore $(n-1)$-dimensional. Let $\phi_1,\dots,\phi_{n-1}$ be a basis for this space. Then they are linearly independent and $uv^T \phi_i = (v\cdot\phi_i)u=0 $. Thus the eigenvalue $0$ has multiplicity $n-1$, and there are no other eigenvalues besides it and $v\cdot u$.
\end{proof}

\begin{proposition}\label{app:prop:lambdaMax}
In the previous lemmas (\ref{app:lem:lambdaMax},  \ref{app:lem:matNorm} and \ref{app:lem:uvt}), assume that $E=yy^*$, then there exists an $0 \leq \epsilon\leq \norm{E}$ such that:
\begin{align}\lambda_{max}(A) + \epsilon = \lambda_{max}(A+yy^*) &\leq \lambda_{max}(A) + \norm{y}^2  \nonumber\\
&= \lambda_{max}(A) + \norm{E} \nonumber
\end{align}
\end{proposition}

\vfill

\end{document}